\documentclass[a4paper,11pt]{article}
\usepackage[margin=1in]{geometry}

\usepackage[colorlinks,linkcolor=blue,citecolor=blue]{hyperref}
\usepackage{amsmath,amssymb,amsthm}
\usepackage[round]{natbib}
\usepackage{mathabx}
\usepackage{graphicx} 
\usepackage{subfigure} 
\usepackage{algorithm,algorithmic}
\usepackage{tabularx}
\usepackage{enumitem}

\usepackage{amsmath,amsfonts,amssymb}
\usepackage{amsthm} 

\usepackage[capitalise]{cleveref}
\crefname{figure}{Figure}{Figures}

\usepackage[table]{xcolor}
\usepackage{tikz}
\usetikzlibrary{arrows,shapes,arrows.meta}

\newcommand{\ignore}[1]{}

\theoremstyle{plain}
\newtheorem{theorem}{Theorem}
\newtheorem{lemma}[theorem]{Lemma}
\newtheorem{corollary}[theorem]{Corollary}

\newtheorem*{theorem*}{Theorem}
\newtheorem*{lemma*}{Lemma}
\newtheorem*{corollary*}{Corollary}
\newtheorem*{proposition*}{Proposition}
\newtheorem*{claim*}{Claim}
\newtheorem*{fact*}{Fact}

\theoremstyle{definition}
\newtheorem{definition}{Definition}
\newtheorem*{definition*}{Definition}

\newtheorem*{remark*}{Remark}

\newtheorem*{example*}{Example}

\theoremstyle{plain}
\newtheorem*{theoremaux}{\theoremauxref}
\gdef\theoremauxref{1}

\DeclareMathAlphabet{\mathbfsf}{\encodingdefault}{\sfdefault}{bx}{n}

\DeclareMathOperator*{\argmin}{arg\,min}

\newcommand{\wt}[1]{\smash{\widetilde{#1}}}

\renewcommand{\O}{O}

\newcommand{\tO}{\wt{\O}}

\newcommand{\E}{\mathbb{E}}

\newcommand{\N}{\mathcal{N}}
\newcommand{\PP}{\mathcal{P}}

\newcommand{\R}{\mathbb{R}}

\let\nablaold\nabla
\renewcommand{\nabla}{\nablaold\mkern-2.5mu}

\newcommand{\Regret}{\mathrm{R}_T}

\newcommand{\rank}{\textrm{rank}}
\newcommand{\erank}{\epsilon\textrm{-rank}}

\title{Online Learning with Many Experts}
\author{
Alon Cohen \qquad Shie Mannor\\
{\normalsize Technion---Israel Institute of Technology}\\
\texttt{\normalsize \{alon.cohen@campus,shie@ee\}.technion.ac.il}
}
 
\begin{document}

\maketitle
\thispagestyle{empty}

\begin{abstract}
We study the problem of prediction with expert advice when the number of experts in question may be extremely large or even infinite. We devise an algorithm that obtains a tight regret bound of $\tO(\epsilon T + N + \sqrt{NT})$, where $N$ is the empirical $\epsilon$-covering number of the sequence of loss functions generated by the environment. 
In addition, we present a hedging procedure that allows us to find the optimal $\epsilon$ in hindsight.

Finally, we discuss a few interesting applications of our algorithm. We show how our algorithm is applicable in the approximately low rank experts model of \citealp{hazan2016online}, and discuss the case of experts with bounded variation, in which there is a surprisingly large gap between the regret bounds obtained in the statistical and online settings.
\end{abstract} 

\section{Introduction}

In this paper we study the well known problem of \emph{prediction with expert advice}, that can be seen as a game between a learner and an environment. At each round $t = 1,2,\ldots,T$, a learner randomly decides to take the advice of an expert $I_t$ from a predetermined set $X$ of experts. Simultaneously, the environment chooses a loss function $\ell_t : X \mapsto [-1,1]$ and afterwards the learner incurs~$\ell_t(I_t)$, the loss associated with that expert.

The goal of the learner throughout the $T$ rounds of the game is to minimize her expected regret, defined as the expected difference between her cumulative loss and the cumulative loss of the best fixed expert in hindsight:
\[
\Regret = \E \left[ \sum_{t=1}^T \ell_t(I_t) - \inf_{i \in X} \sum_{t=1}^T \ell_t(i) \right]~,
\]
where the expectation is taken over the random choices of the learner.

A fundamental result in the field of online learning states that, in the worst case, the best strategy for the learner incurs $O(\sqrt{T \log K})$ regret in the worst case \citep{cesa1997use}, where $K$ is the number of experts. However, in many natural problems the number of experts may be extremely large, possibly infinite, albeit their advices may be highly "correlated". In an extreme case, their advices may even be grouped together in a small number of clusters.

As a motivating application consider the problem of investment portfolio selection over a given set of stocks. These stocks may be categorized by a small number of parameters and each parameter may have a small number of possible values, but the overall number of combinations of these parameters can be very large. Thus every expert may advise on a different portfolio, but the portfolios themselves may give similar revenues due to the correlation resulted by the overlap in parametrization between the different stocks. 

Another possible application is when the experts themselves arise from a reduction between another online learning problem and prediction with expert advice. This can occur, for example, in online supervised learning \citep{ben2009agnostic}, adaptive algorithms \citep{van2016metagrad} and algorithms for tracking problems \citep{gyorgy2005tracking}, in which the resulting number of experts can easily become exponentially large in the parameters of the problem.

Our goal is to take advantage of such structure in order to achieve lower regret. In particular, we hope to be able to obtain regret bounds that are independent of the number of experts. We do so by constructing a cover of the sequence of loss functions generated by the environment. Namely, we find a small set of experts $S$ such at least one of them has a similar cumulative loss as the best expert in~hindsight.

Our results are motivated by the stochastic case, in which the loss functions are sampled i.i.d.~from a fixed distribution. In this case, a vast literature in the field of statistical machine learning shows that the regret is controlled by the covering number of the problem, giving a regret rate of $O(\epsilon T + \sqrt{T \log N})$ where $N$ is the $\epsilon$-covering number \citep{shalev2014understanding}. Additionally, this bound is attained by a simple successive ERM algorithm that is oblivious to the structure of the problem \citep{hazan2016online}.

Quite disappointingly, this is hardly the case in online learning, as even in extremely simple cases it is not possible to achieve logarithmic dependence on the covering number. This is due to the fact that as the game progresses the learner reveals more and more about the structure of the experts. However at any point in the game, this structure may develop in an exponentially large number of ways, depending on the whim of the environment and unbeknownst to the learner. In fact, in this setting any learner cannot do better that $\Omega(T)$ regret in the worst case (for example, if the number of experts is exponential in $T$), and yet against benign environments we can expect better performance.
 
\subsection{Our contributions}

In this work we study the problem of prediction with expert advice when the number of experts is very large and even possibly infinite. We show an online learning algorithm that can obtain a regret bound of $\tO(\epsilon T + \N(\epsilon, L_T) + \sqrt{T \N(\epsilon, L_T)})$ for a parameter $\epsilon$, where $L_T = (\ell_1,\ell_2,\ldots,\ell_T)$ is the sequence of loss functions generated by the environment and $\N(\epsilon, L_T)$ is the $\epsilon$-covering number of $L_T$ under the infinity-norm.
The algorithm does so by iteratively constructing a packing of the experts on $L_T$ as it is gradually revealed to the learner.

Additionally, we regard several applications and extensions of our algorithm. We explain how to find the optimal $\epsilon$ automatically without any prior knowledge and discuss a few applications of our algorithm. In particular, we discuss the case of binary losses and an application of our algorithm to the low rank expert model. We also consider the interesting case of experts with bounded variation, in which it is possible to achieve a regret bound independent on the number of experts in the stochastic case. Surprisingly, the case of the online setting is drastically different.

\subsection{Related work}

Covering and packing \citep{rogers1964packing} in compact metric spaces is a common discretization tool used in many fields of statistics. These include machine learning \citep{vapnik1998statistical}, empirical processes \citep{pollard1990empirical} and information theory \citep{roman1992coding}.
Additionally, papers in a few different topics are related to ours.

\paragraph{Online combinatorial optimization.}

Online combinatorial optimization \citep{koolen2010hedging} is a subset of online learning in which the learner's predictions form different kinds of combinatorial objects. These include the well known online shortest path problem \citep{takimoto2003path}, permutations \citep{helmbold2009learning} and many more.

For example, in the online shortest path problem, each expert corresponds to a path in a predetermined graph. At each round the environment sets losses to the edges of the graph, and the loss that corresponds to a path is simply the sum of the losses along the edges of the path.

In this case, even though the number of experts may be exponential in the number of vertices and edges of the graph, the losses of the experts are well structured. Indeed, if two paths share many of their edges then we can expect their losses to be similar.

\paragraph{Quantile bounds.}

Quantile bounds \citep{chaudhuri2009parameter,chernov2010prediction} are bounds of the form $O(\sqrt{T \log (1/\epsilon)})$ on the regret against the worst of a $1/\epsilon$-fraction of the "leading" experts, in which by leading we mean experts with least cumulative loss. However, these methods cannot guarantee nontrivial regret against \emph{every} expert without sufficient prior knowledge over the~environment.\footnote{More specifically, if the experts can be embedded into some function space, one needs to know the embedding a-priori.}

\paragraph{Branching experts.}

A conceptually similar setting to ours is that of branching experts \citep{gofer2013regret}, in which the learner starts from a small number of experts, and where at any point in the game every expert can split into an arbitrary number of experts. This, for example, can be used to model a setting in which $K$ potential experts are in fact only $k$ distinct experts, but these $k$ are not known by the learner in advance. In this case the authors prove that a regret of $\Theta(\sqrt{k T})$ is tight (assuming $k < \log_2 N$). Note that this is comparable to $\Theta(\sqrt{T \log k})$ obtained in the stochastic setting.

This setting is different to ours since: first, they assume that the number of experts is finite but may be very large; Secondly, the experts themselves branch whenever their losses differ, whereas our algorithm is oblivious to this information. All our algorithm needs to know is if there is some expert that is not covered by our current packing. However, note that their lower bound of~$\Omega(\sqrt{kT})$ is still applicable in our case.

\paragraph{Low rank experts.}

\cite{hazan2016online} consider the case where the losses of the experts reside in some unknown $d$-dimensional linear subspace. This is a special case of our setting in which the learning algorithm is much simpler and more intuitive. In particular they show a $O(d \sqrt{T})$ regret bound in this setting.

\paragraph{Simulatable experts.} 

In the simulatable experts model \citep{cesa1999prediction}, the learner knows in advance the advices of all of the experts, as she attempts to learn a binary sequence chosen by the environment. At each round, the learner predicts a binary value at random and suffer a loss according to the absolute loss. 

This model is one of improper learning, in the sense that the learner predicts a binary bit at each round rather than following the advice of a single expert. This allows the authors to obtain a tight characterization of the regret in terms of the Rademacher complexity of the expert class. As such, their bounds are very different than ours.

\paragraph{Sequential Rademacher complexity.}

Sequential Rademacher complexity \citep{rakhlin2010online} is a useful tool for obtaining regret bounds. However, unlike the classic Rademacher Complexity of statistical learning theory, it is often difficult to use. This is added to the fact that the regret bounds obtained using this method are non-algorithmic in nature.

\paragraph{Adaptive online algorithms.} 

Most of the work previously done in prediction with expert advice revolves around attempting to improve the dependence on $T$ under different assumptions.
These include algorithms that can adapt to data that varies slowly \citep{cesa2007improved,chiang2012online,hazan2010extracting,rakhlin2013online}, as well as algorithms that can adapt to stochastic i.i.d. losses \citep{de2014follow,sani2014exploiting}.
However, in our work we aim to improve the dependence on $K$ even at the cost of possibly hindering the dependence on $T$.

\section{Preliminaries and Main Results}

In this section we provide some preliminary information, discussing the Exponential Weights algorithm as well as defining the notions of covering and packing. Thereafter, we give our main~results.

\subsection{Exponential weights}
\begin{algorithm}
\caption{Exponential Weights \label{alg:expw}}
\begin{algorithmic}
\STATE \textbf{Parameters}: Number of experts $K$.
\STATE \textbf{Set}: $w_1(i) = 1$ for all $i = 1,2,\ldots,K$.
\FOR{t = 1,2,\ldots}
\STATE Define distribution $p_t$ by $p_t(i)~\propto~w_t(i)$.
\STATE Predict $I_t \sim p_t$ and suffer loss $\ell_t(I_t)$.
\STATE \textbf{Set}: $\eta_t = \sqrt{8 \log(K) / t}$.
\STATE \textbf{Update}: $w_{t+1}(i) = w_t(i) \exp(-\eta_t \ell_t(i))$ for all $i = 1,2,\ldots,K$.
\ENDFOR
\end{algorithmic}
\end{algorithm}
%
%
Exponential Weights \citep{littlestone1989weighted,vovk1995game,cesa1997use}, also named Hedge and Randomized Weighted Majority, is a celebrated algorithm for prediction with expert advice for a finite number of experts. The variant that we give here, depicted in \cref{alg:expw}, has an adaptive learning rate and therefore the algorithm does not need to know the length of the game $T$ in advance.

The algorithm assigns a weight for each expert that is initially set to 1. In each round, the algorithm chooses an expert at random from a distribution that is proportional to the weights of the experts, after which the weight of each expert is decreased according to her loss. 

We have the following guarantee on the regret of the algorithm.

\begin{lemma}[{\citealp[Theorem 2.3]{cesa2006prediction}}]
\label{lemma:expwanalaysis}
The expected regret of \cref{alg:expw} satisfies~$\Regret \le 4\sqrt{T \log K}$.
\end{lemma}

\subsection{Covering and packing}

A cover of a sequence of loss functions is a small finite subset of experts, such that, intuitively, for any expert we can find an expert in the cover with similar losses.
\begin{definition}[Cover]
An $\epsilon$-cover of a sequence of loss functions $L_T = (\ell_1,\ell_2,\ldots,\ell_T)$ is a subset of experts $S$ that satisfies the following: 
for every expert $i$ (not necessarily in $S$) there is an expert $j \in S$ such that for all $t = 1,2,\ldots,T$ we have $|\ell_t(i) - \ell_t(j)| \le \epsilon$.
The $\epsilon$-\emph{covering number} of $L_T$, denoted $\N(\epsilon, L_T)$, is the size of the smallest $\epsilon$-cover of $L_T$.
\end{definition}
An important implication of the definition is that the covering number is monotonically decreasing in $\epsilon$.
The following definition is in some sense the dual of a cover. 
%
\begin{definition}[Packing]
An $\epsilon$-packing of a sequence of loss functions $L_T = (\ell_1,\ell_2,\ldots,\ell_T)$ is a subset of experts $S$ that satisfies the following: 
for every different experts $i, j \in S$ there is a $t \in [T]$ such that $|\ell_t(i) - \ell_t(j)| > \epsilon$.
The $\epsilon$-\emph{packing number} of $L_T$, denoted $\PP(\epsilon, L_T)$, is the size of the largest $\epsilon$-packing of $L_T$.
\end{definition}

This next lemma is a well known result about duality between covering and packing, however for completeness we shall provide its proof in \cref{sec:proofdualitylemma}. 

\begin{lemma}
\label{lemma:coverpackduality}
For any sequence of loss functions $L_T$ we have $\PP(2 \epsilon, L_T) \le \N(\epsilon, L_T) \le \PP(\epsilon, L_T)$.
\end{lemma}

\subsection{Main results}
We now state the main results of the paper.
Our first result shows that there is an algorithm that obtains a regret bound that depends on the empirical covering number of the sequence of loss functions generated by the environment, with no direct dependence on the number of~experts.

\begin{theorem}
\label{thm:alganalysis}
Let $0< \epsilon \le 1$. Suppose that the environment generates a sequence of loss functions $L_T = (\ell_1,\ell_2,\ldots,\ell_T)$.
\cref{alg:expinf} described in \cref{sec:algorithm} attains an expected regret~of 
\[
\Regret = \tO \left(\epsilon T + \N(\epsilon, L_T) + \sqrt{T \N(\epsilon, L_T)} \right)
\]
\end{theorem}

Unlike \cref{alg:expw}, our algorithm is applicable even when the number of experts is infinite. Nonetheless, note that if the number of experts is, say $K$, then the bound above becomes meaningful when $\N(\epsilon, L_T) < \log K$. Finally, the bound of \cref{thm:alganalysis} is tight due to a matching lower bound found in \citet[Lemma 10]{gofer2013regret}.

Our next result handles the case where the optimal accuracy $\epsilon$ is not known in advance. Luckily, there is a simple procedure that can guarantee the same regret bound as if the optimal $\epsilon$ is known.

\begin{corollary}
\label{corr:tuneepsilon}
Suppose that the environment generates a sequence $L_T = (\ell_1,\ell_2,\ldots,\ell_T)$ of loss functions. There exists an online learning algorithm, described in \cref{sec:tune}, whose regret is bounded as follows.
\[
\Regret = \inf_{0 < \epsilon \le 1} \tO \left( \epsilon T + \N(\epsilon, L_T) + \sqrt{T \N(\epsilon, L_T) }\right)
\]
\end{corollary}

Additionally, in \cref{sec:applications} we consider a number of applications of our algorithm.
We remark on the case of binary losses, we show an application of our algorithm to the setting of low rank experts, discuss the case of losses from a sparse dictionary and consider the case of experts with bounded variation. 
For a sequence of loss functions $\ell_1,\ell_2,\ldots,\ell_T$, let us define the variation of expert $i$ as
\[
V(i) = \sum_{t=1}^{T-1} \left|\ell_{t+1}(i) - \ell_t(i) \right|~.
\]
We show that when the variation of all experts is at most $V$, even though in the stochastic i.i.d case it is possible to obtain $O(\sqrt{VT})$ regret, this hardly the case in the online setting. In particular, even if $V = O(1)$ and the losses are binary --- the loss of every expert only changes once during the game --- the regret still grows with the number of experts in the worst case.

\section{Algorithm}
\label{sec:algorithm}
Our algorithm is depicted in \cref{alg:expinf}.
\begin{algorithm}
\caption{Exponential Weights for Many Experts \label{alg:expinf}}
\begin{algorithmic}
\STATE \textbf{Parameters}: Number of rounds $T$, accuracy $\epsilon \in (0,1]$.
\STATE \textbf{Set}: $\tau_1 = 0, r = 1, K_1 = 1, w_1(1) = 1$ and let $S$ contain an arbitrary expert.
\FOR{t = 1,2,\ldots,T}
\STATE Define distribution $p_t$ by $p_t(i)~\propto~w_t(i)$.
\STATE Predict $I_t \sim p_t$ and suffer loss $\ell_t(I_t)$.
\WHILE{some expert $j$ has $|\ell_t(j) - \ell_t(i)| > 2 \epsilon$ for all $i \in S$}
\STATE Add $j$ to $S$.
\ENDWHILE
\IF{experts were added during this round}
\STATE Let $K_{r+1} = |S|$ be current number of experts.
\STATE \textbf{Set}: $w_{t+1}(i) = 1$ for all $i \in S$ (perform a restart), $\tau_{r+1} \gets t, r \gets r+1$.
\ELSE
\STATE \textbf{Set}: $\eta_t = \sqrt{8 \log(K_r) / (t - \tau_r)}$.
\STATE \textbf{Update}: $w_{t+1}(i) = w_t(i) \exp(-\eta_t \ell_t(i))$ for all $i \in S$.
\ENDIF
\ENDFOR
\STATE \textbf{Set}: $\tau_{r+1} = T+1$.
\end{algorithmic}
\end{algorithm}
The algorithm starts with a set $S$ containing one expert, and gradually builds a $2 \epsilon$-packing of experts on the sequence of loss functions generated by the environment. Namely, whenever there is an expert whose loss is more that $2 \epsilon$ away from all of the experts in $S$, the algorithm adds her to $S$. 

This produces two guarantees for us. The first, is that when $S$ is not updated, the loss of any expert not in $S$ is at most $2 \epsilon$ apart from the loss of one of the experts in $S$. Second, as our regret bound will depend on the size of $S$, \cref{lemma:coverpackduality} entails that by the end of the game, the size of $S$ is at most that of an $\epsilon$-cover of the sequence.

Additionally, when $S$ is not updated the algorithm behaves exactly the same as \cref{alg:expw}. Nonetheless whenever $S$ is updated, the algorithm performs a restart: it resets the weights of the experts as well as the learning rate $\eta_t$ accordingly.

\subsection{Analysis}
\begin{proof}[Proof of \cref{thm:alganalysis}]
First note that \cref{alg:expinf} acts in $p$ phases between restarts. This means that during each phase the algorithm behaves exactly like \cref{alg:expw} with the $K_r = |S|$ chosen~experts.

For any expert $i$, consider the sequence $i_1,i_2,\ldots,i_p$ of experts that cover $i$ in each of the phases, namely within each phase $r$ we have $|\ell_t(i_r) - \ell_t(i)| \le 2 \epsilon$. Note that $T_r = \tau_{r+1} - \tau_r$ is the length of the $r$'th phase, then by \cref{lemma:expwanalaysis} the regret of the algorithm during this phase with respect to $i_r$ is:
\[
\mathrm{R}_r \le 4 \sqrt{T_r \log K_r}~.
\]

Additionally, during the phase the loss of expert $i$ is at most $2 \epsilon$ away from that of expert $i_r$, and in-between phases we have a single round in which the instantaneous regret is at most~2. Therefore, the regret of \cref{alg:expinf} with respect to $i$ is bounded as follows:
\begin{align}
\E \left[\sum_{t=1}^T \ell_t(I_t) - \ell_t(i) \right] &= \E \left[ \sum_{r=1}^p \sum_{t=\tau_r + 1}^{\tau_{r+1}-1} \ell_t(I_t) - \ell_t(i) + \sum_{r=2}^p \underbrace{\ell_{\tau_r}(I_t) - \ell_{\tau_r}(i)}_{\le 2} \right] \nonumber \\
&\le \sum_{r=1}^p \E \left[ \sum_{t=\tau_r + 1}^{\tau_{r+1}-1} \ell_t(I_t) - \ell_t(i) \right] + 2p~, \label{eq:algregret}
\end{align}
and for each $r$ we have
\begin{align*}
\E \left[ \sum_{t=\tau_r + 1}^{\tau_{r+1}-1} \ell_t(I_t) - \ell_t(i) \right] &= \underbrace{\E \left[ \sum_{t=\tau_r + 1}^{\tau_{r+1}-1} \ell_t(I_t) - \ell_t(i_r) \right]}_{ = \mathrm{R}_r} + \sum_{t=\tau_r + 1}^{\tau_{r+1}-1} \underbrace{\ell_t(i_r) - \ell_t(i)}_{\le 2 \epsilon} \nonumber \\
&\le \mathrm{R}_r + 2 \epsilon (\tau_{r+1} - \tau_r)~. 
\end{align*}
Summing over all $r = 1,2,\ldots,p$ we get 
\begin{align}
\sum_{r=1}^p \mathrm{R}_r + \epsilon (\tau_{r+1} - \tau_r) &\le \sum_{r=1}^p 2 \sqrt{T_r \log K_r} + 2 \epsilon (\tau_{p+1} - \tau_1) \nonumber \\
&\le 4 \sqrt{\sum_{r=1}^p T_r} \sqrt{ \sum_{r=1}^p \log K_r}  + 2 \epsilon T \nonumber \\
&\le 8 \sqrt{T K_p \log K_p}  + 2 \epsilon T~, \label{eq:algregretsum}
\end{align}
where the second inequality is by the Cauchy-Schwartz inequality and since $K_1,K_2,\ldots,K_p$ is an increasing sequence. The third inequality is since $\sum_{r=1}^p T_r \le T$ and since $\sum_{r=1}^p \log K_r \le 2 K_p \log K_p$ by \cref{lemma:logsum} (technical).

Finally, we notice that the $K_p$ experts in $S$ at end of the game form a $2\epsilon$-packing of the sequence $\ell_1,\ell_2,\ldots,\ell_T$. Indeed, let $i,j \in S$ be two experts, suppose that $i$ is added to $S$ before $j$, and let $t$ be the round in which $j$ is added to $S$. Then by the definition of the algorithm, we must have $|\ell_t(i) - \ell_t(j)| > 2 \epsilon$. Thus by \cref{lemma:coverpackduality} we have $K_p \le N$. In addition, whenever the algorithm performs a restart it adds at least one expert to $S$, and therefore we have $p \le K_p$.
Combining these facts with \cref{eq:algregret} and with \cref{eq:algregretsum} gives the desired result.
\end{proof}

\section{Tuning $\epsilon$ Automatically}
\label{sec:tune}

In this section we prove \cref{corr:tuneepsilon}.
Suppose that we do not know what the optimal $\epsilon$ is in advance. Looking at the regret bound of \cref{thm:alganalysis}, we have a tradeoff --- choosing a smaller $\epsilon$ may decrease the $\epsilon T$ term but may increase the covering number $\N(\epsilon, L_T)$. We would like to tune $\epsilon$ to be the best possible in hindsight. 

In this case we can run $R = \lceil \log_2 T \rceil$ copies of our algorithm with exponentially decreasing accuracy parameters; for each algorithm $r = 1,2,\ldots,R$ we set $\epsilon_r = 2^{-r+1}$. By treating these $R$ algorithms themselves as experts, we can run a copy of \cref{alg:expw}, such that whenever it chooses an algorithm $r$, we play the action chosen by algorithm $r$ on this round. 
For this procedure we have the following analysis.

\begin{proof}[Proof of \cref{corr:tuneepsilon}]
Let $L_r$ be the cumulative loss of algorithm $r$, let $L$ be the cumulative loss of our procedure and $L^\star$ be the cumulative loss of the best expert in hindsight. Then by \cref{lemma:expwanalaysis} and \cref{thm:alganalysis} the regret of this procedure is bounded as follows.
\begin{align*}
\E \left[ L - L^\star \right] &= \E \left[ \left(L - \min_{r \in [R]} L_r \right) + \left( \min_{r \in [R]} L_r - L^\star \right) \right] \\
&\le \E \left[ L - \min_{r \in [R]} L_r \right] + \min_{r \in [R]} \E \left[ L_r - L^\star \right] \\
&\le 4 \sqrt{T \log \log T}~+ \\
&\qquad \min_{r \in [R]} \left\lbrace 4 \cdot 2^{-r} T + 2 \N(2^{-r+1}, L_T) + 8 \sqrt{T \N(2^{-r+1}, L_T) \log \N(2^{-r+1}, L_T)} \right\rbrace~,
\end{align*}
where the first inequality is due to Jensen's inequality.
To complete the proof, we will show that
\begin{align*}
&\min_{r \in [R]} \left\lbrace 4 \cdot 2^{-r} T + 2\N(2^{-r+1}, L_T) +  8 \sqrt{T \N(2^{-r+1}, L_T) \log \N(2^{-r+1}, L_T)} \right\rbrace \\
&\qquad \le \inf_{0<\epsilon\le 1} \left\lbrace 4 \epsilon T + 2\N(\epsilon, L_T) + 8 \sqrt{T \N(\epsilon, L_T) \log \N(\epsilon, L_T)} \right\rbrace + 2~.
\end{align*}
Indeed, consider any $\epsilon \in (0,1]$. If $\epsilon > 2^{-R+1}$, then let $r^\star$ be maximal such that $2^{-r^\star+1} \ge \epsilon$. In particular we have $2^{-r^\star} < \epsilon$. We get
\begin{align*}
&\min_{r \in [R]} \left\lbrace 4 \cdot 2^{-r} T + 2\N(2^{-r+1}, L_T) + 8 \sqrt{T \N(2^{-r+1}, L_T) \log \N(2^{-r+1}, L_T)} \right\rbrace \\
&\qquad \le 4 \cdot 2^{-r^\star} T + 2\N(2^{-r^\star+1}, L_T) + 8 \sqrt{T \N(2^{-r^\star+1}, L_T) \log \N(2^{-r^\star+1}, L_T)} \\
&\qquad \le 4 \epsilon T + 2\N(\epsilon, L_T) + 8 \sqrt{T \N(\epsilon, L_T) \log \N(\epsilon, L_T)}~,
\end{align*}
where we have used the fact that the covering number is monotonically decreasing in $\epsilon$.
On the other hand, if $\epsilon \le 2^{-R+1}$ then
\begin{align*}
&\min_{r \in [R]} \left\lbrace 4 \cdot 2^{-r} T + 2\N(2^{-r+1}, L_T) + 8 \sqrt{T \N(2^{-r+1}, L_T) \log \N(2^{-r+1}, L_T)} \right\rbrace \\
&\qquad \le 4 \cdot 2^{-R} T + 2\N(2^{-R+1}, L_T) + 8 \sqrt{T \N(2^{-R+1}, L_T) \log \N(2^{-R+1}, L_T)} \\
&\qquad \le 2 + 2\N(\epsilon, L_T) + 8 \sqrt{T \N(\epsilon, L_T) \log \N(\epsilon, L_T)}~,
\end{align*}
thus reaching the desired conclusion.
\end{proof}

\section{Applications}
\label{sec:applications}

In this section we discuss a number of applications of \cref{alg:expinf}. First, we discuss the case of binary losses in which we can attain a regret bound that depends on the number of experts with distinct sequences of losses. We then discuss an application of our algorithm to the low rank experts model and a generalization of it in the form of experts whose losses are acquired from a sparse dictionary. Finally, we approach the interesting case of experts with bounded variation, in which we show a large gap between the regret of the stochastic and the online~settings.

\subsection{Binary losses}

Consider a setting in which the losses of the experts are binary, namely -1 or 1. In this case, we can obtain a regret bound that depends on the number experts with distinct sequences of losses, without paying for an additional term that depends linearly on $T$. 
The following is a direct consequence of \cref{thm:alganalysis}.

\begin{corollary}
\label{corr:binary}
Let $N = \N(0, L_T)$ be the $0$-covering number of sequence $L_T = (\ell_1,\ell_2,\ldots,\ell_T)$ of binary loss functions generated by the environment. Then the regret of the \cref{alg:expinf} satisfies
\[
\Regret \le N + 8 \sqrt{T N \log N}~.
\]
\end{corollary}
%
%

\subsection{Low rank experts}

Suppose that the number of experts $K$ is finite but very large. Let $L \in [-1,1]^{T \times K}$ be the losses arranged in a matrix obtained in hindsight, and consider the model of \cite{hazan2016online} in which the environment plays a strategy in which the $\epsilon$-approximate rank of $L$ is $d$. 
The approximate rank of a matrix $L$ is defined as~follows:
\[
\erank(L) = \min \left\lbrace \rank(L') : \| L' - L \|_\infty \le \epsilon \right\rbrace~,
\]
namely it is the lowest rank of any matrix that $\epsilon$-approximates $L$ entry-wise.

Under this setting, \cite{hazan2016online} give an online learning algorithm that gives $O(d\sqrt{T})$ regret if $L$ is of (0-)rank $d$. For the case of $\epsilon$-approximate rank, the authors give a regret bound of $O(\sqrt{dT} + \epsilon \sqrt{T \log K})$ for the stochastic case, and leave the problem of obtaining a similar bound in the online case as an open issue. 
In this section, we show an application of our algorithm to this latter setting, described by the following corollary.
\begin{corollary}
\label{corr:lowrankcover}
Let $0 < \epsilon \le 1/4$. Suppose that $T \le K$ and $\erank(L) \le d$, then \cref{alg:expinf} applied to this setting with accuracy $4\epsilon$ gives a regret bound of 
\[
\Regret = \tO \left(\epsilon T + (c / \epsilon)^d + \sqrt{(c/\epsilon)^d T} \right)~,
\]
for an absolute constant $c$.
\end{corollary}
Note that the bound above is only meaningful if it happens that $\omega(T^{-1/d}) \le \epsilon \le o(1)$. 
The bound hints on an exponential decay in the dependence on the approximate rank $d$ between the stochastic and online cases. Whether this gap can be removed using nontrivial algorithmic techniques remains an open issue and an interesting direction for future research.

Let us turn to prove \cref{corr:lowrankcover}, but in order to do so we shall need the following theorem.
\begin{theorem}[{\citealp[Theorem 3.2]{alon2013approximate}}]
\label{thm:alon}
Let $A$ be an $K \times K$ matrix with entries in $[-1, 1]$ and $\erank(A) = d$. 
Let $\Delta$ be the $(K-1)$-dimensional probability simplex. 
There is a finite set $S \subseteq \R^K$ such that
\[
\forall x \in \Delta,~\exists \tilde{x} \in S : \|Ax - A \tilde{x}\|_\infty \le 2 \epsilon~,
\]
and $|S| = O(1/ \epsilon)^d$. 
\end{theorem}

We shall now continue with the proof of the corollary.
\begin{proof}[Proof of \cref{corr:lowrankcover}]
Consider an application of \cref{alg:expinf} with accuracy $4 \epsilon$ in our problem, for which by \cref{thm:alganalysis} it suffices to bound the covering number $\N(4 \epsilon, L_T)$.

Consider padding the loss matrix $L \in [-1,1]^{T \times K}$ with $K-T$ rows of zero entries, in order to obtain a $K \times K$ matrix whose $\erank$ remains $d$. Since we can represent each expert $i$ by a standard basis vector $e_i \in \R^K$, the vector $L e_i$ represents the losses of expert $i$.
Thus, we can apply \cref{thm:alon} to $L$ and acquire a set $S \subseteq \R^K$ of size $O(1/\epsilon)^d$ with the following property:
\[
\forall i \in [K],~\exists \tilde{x} \in S : \|L e_i - L \tilde{x}\|_\infty \le 2 \epsilon~.
\]

Let $R$ be a maximal $4 \epsilon$-packing of $L_T$. By \cref{lemma:coverpackduality} we have $\N(L_T, 4 \epsilon) \le |R|$ so that it suffices to show a one-to-one mapping $\pi : R \mapsto S$, that would imply $|R| \le |S|$.
Indeed, define $\pi$ as follows: for each $i \in R$ let $\pi(i) = \arg\min_{\tilde{x} \in S} \| L e_i - L \tilde{x}\|_\infty$ breaking ties arbitrarily. 
To show that it is one-to-one, let $i,j \in R$ such that $\pi(i) = \pi(j)$, then
\[
\| L e_i - L e_j \|_\infty \le \| L e_i - L \pi(i) \|_\infty + \| L \pi(j) - L e_j \|_\infty \le 2 \epsilon + 2 \epsilon = 4 \epsilon~.
\]
Since $R$ is a $4 \epsilon$-packing it must be the case that $i = j$, as required.
\end{proof}

\subsection{Losses from a sparse dictionary}

Sparse dictionary learning \citep{elad2006image}, also known as sparse coding, is a widely used tool in machine learning, neuroscience and signal processing. Given an unlabeled dataset, this method approximates each data point by a linear combination of a small number of vectors from a set, called a dictionary. This modeling is motivated by the empirical success \citep{aharon2006rm,lee2007efficient}, as well as evidence that, for example, the neurons of the V1 optical cortex use similar representations \citep{olshausen1997sparse}. In the following we assume that the environment plays a strategy in which the losses of the experts can be approximated by such a sparse~representation.

Formally, let the environment play a strategy that satisfies the following. Suppose that the loss matrix $L \in [-1,1]^{T \times K}$ can be approximated by a decomposition $D \cdot V$, namely that it satisfies $\|L - D \cdot V\|_\infty \le \epsilon$. Here $D$, the dictionary, is a $T \times n$ matrix whose rows have 1-norm of at most 1. The matrix $V$ is an $n \times K$ matrix in which each column $i$, associated with expert $i$, is a $k$-sparse\footnote{A $k$-sparse vector is one that has at most $k$ nonzero entries.} vector $v_i$ such that $\|v_i\|_\infty \le 1$. Note that both matrices $D$ and $V$ are unknown to the learner and chosen by the environment in an adversarial manner.

This assumption entails that there is an approximate sparse representation for the losses of each expert. Indeed, the vector of losses of expert $i$ is $\epsilon$-approximated by $D v_i$, and since $v_i$ is $k$-sparse, $D v_i$ is a linear combination of at most $k$ of the columns of $D$.
The dictionary $D$ can be over-complete, which means that $n \ge T$ and that its columns are not necessarily orthogonal. This allows the vectors $D v_i$ to lie in different, yet possibly overlapping, $k$-dimensional subspaces of $\R^T$. As such this setting is a generalization of the low rank experts model of the previous section.

In this setting we have the following guarantee for our algorithm.
\begin{corollary}
\label{corr:dict}
Suppose that $0<\epsilon \le 1/4$. \cref{alg:expinf} applied to the setting above with accuracy $4\epsilon$ gives a regret bound of 
\[
\Regret = \tO \left(\epsilon T + (2n / \epsilon)^k + \sqrt{(2n/\epsilon)^k T} \right)~.
\]
\end{corollary}
In particular, note that the bound is meaningful when $\omega(nT^{-1/k}) \le \epsilon \le o(1)$.

\begin{proof}[Proof of \cref{corr:dict}]
Once again, in view of \cref{thm:alganalysis} it suffices to bound the covering number~$\N(4\epsilon, L_T)$.
As in the proof of \cref{corr:lowrankcover}, associate with each expert $i$ the standard basis vector $e_i \in \R^K$ and recall that $L e_i \in [-1,1]^T$ is the vector of losses associated with expert $i$.
Let $i,j$ be any two experts, then we have
\begin{align*}
\|Le_i - L e_j \|_\infty &\le \underbrace{\| L e_i - D v_i \|_\infty}_{\le \epsilon} + \underbrace{\| D v_i - D v_j \|_\infty}_{\le \|v_i - v_j\|_\infty} + \underbrace{\|D v_j - L e_j \|_\infty}_{\le \epsilon} \\
&\le \| v_i - v_j \|_\infty + 2 \epsilon~,
\end{align*}
where the last inequality is by our assumption that the rows of $D$ have 1-norm of at most 1. 
Therefore, it suffices to $2\epsilon$-cover $S$, the set of all $k$-sparse vectors $v$ that satisfy $\|v\|_\infty \le 1$ .

The set $S$ is a union of $\binom{n}{k}$ $k$-dimensional cubes of the form $[-1,1]^k$. By a well known result on covering numbers, each such cube can be covered by at most $(1 + 1/\epsilon)^k \le (2/\epsilon)^k$ cubes of the form $[-2\epsilon,2\epsilon]^k$.
This entails that $\N(\epsilon, L_T) \le \binom{n}{k} (2/\epsilon)^k \le (2 n / \epsilon)^k$, and plugging the latter into the bound \cref{thm:alganalysis} gives the desired result.
\end{proof}

\subsection{Experts with bounded variation}
\label{sec:boundedvar}

For this section, let us assume that the variation of all of the experts is bounded by $V$.
Well known results about the covering numbers of functions of bounded variation \citep{bartlett1994fat} show that in the stochastic i.i.d case, it is possible to learn such functions while suffering only $\tO(\sqrt{VT})$ regret. However, the situation is drastically different in the online case.

For motivation consider the case of binary losses, then $V/2$ is a bound on the number of times the loss of an expert changes from -1 to 1 or vice versa, during the $T$ rounds of the game. Therefore, the number of experts with distinct losses is $2 \sum_{i=0}^{V/2} \binom{T-1}{i} = O(T^{V/2})$. This entails that by \cref{corr:binary}, we have the a  guarantee of $\tO( T^{V/2})$ on the regret of \cref{alg:expinf}, which is trivial even for $V = 2$! 
This leaves us with the following problem: does there exists an algorithm for online learning that attains $\tO(\sqrt{VT})$ regret when the variation of all experts is bounded by~$V$? 

We answer this question negatively. 
First note that the interesting regime is when $T < \log K$, otherwise we can obtain a tight $\Theta(\sqrt{V \log K})$ bound of \cite{hazan2010extracting}. However, if $T < \log K$, even if $V = O(1)$ we cannot expect the desired regret bound, as shown in the following~result.

\begin{theorem}
Suppose $T < \log_2 K$ and that $V(i) \le 2$ for every expert $i$. Then there is a randomized environment that forces any learner to obtain a regret of at least $T$ in expectation.
\end{theorem}

\begin{proof}
Our construction is as following. We start from setting the loss of all experts to -1. At each round, we pick a subset of experts whose loss was -1 so far, set it to 1 and it will stay 1 for the remainder of the game. This will make sure that the variation of all experts is indeed at most 2.

On the first round we pick half of the experts uniformly at random and set their loss to 1. Therefore, on the first round, the loss of any learner is exactly 0 in expectation. In the second round, we pick half of the experts whose loss was -1 in the first round and set their loss to 1. Once again, any learner must suffer a loss of at least 0 in expectation on the second round. We keep doing so for the $T$ rounds of the game.

Thus, the expected regret of the learner is at least $T$ since her expected loss is at least 0, but there is at least one expert with a cumulative loss of $-T$.
\end{proof}

\section{Discussion and Open Problems}

In this work we have shown an algorithm that obtains a regret bound that is independent of the number of experts. This dependence is replaced by a certain covering number that governs the complexity of the observed sequence of loss functions.
We have also shown how to automatically tune the accuracy parameter of our algorithm. Finally, we have presented a number of applications of our algorithm, that include binary losses, low rank experts and experts with bounded variation.

Unfortunately for many important applications, the covering number used by our algorithm can typically be very large. Given that ideally we would like to choose $\epsilon = o(1)$, it is an interesting direction for future work to try and to find a simple way to quantify which classes of environments produce small covering numbers for such values of $\epsilon$.
In addition, comparing our result to that of \cite{hazan2016online} in their setting, our bound is exponentially worse in the dimension of the loss matrix. It remains an open problem to find an algorithm that has a tight regret bound in both cases. 

Another interesting direction to explore is the connection between our setting and online compression. Indeed, by setting the losses into a matrix $L \in [-1,1]^{T \times K}$ our problem is equivalent to approximating this matrix with a small number of columns in an online fashion. Therefore any guarantee provided by an online compression algorithm that solves this problem, implies an improvement in guarantee over the regret.

%

%


\bibliographystyle{abbrvnat}
\bibliography{bib}

\appendix

\section{Additional Proofs}

\subsection{Proof of \cref{lemma:coverpackduality}}
\label{sec:proofdualitylemma}

\begin{proof}[Proof of \cref{lemma:coverpackduality}]
For the lower bound let $S$ be a $2 \epsilon$-packing of $\ell$ and $V$ be an $\epsilon$-cover of $\ell$. 
Define the following distance function between experts:
\[
d(i,j) = \max_{t \in [T]} |\ell_t(i) - \ell_t(j)|~.
\]
Now, we define a mapping $\pi : S \mapsto V$ by $\pi(i) = \argmin_{j \in V} d(i,j)$ (breaking ties arbitrarily), namely for every expert $i \in S$ we take $\pi(i)$ to be the expert in $V$ that is closest to $i$ according to $d$.

Note that it suffices for us to show that $\pi$ is one-to-one as this will show that $|S| \le |V|$. 
Indeed, suppose $i,j \in S$ are such that $\pi(i) = \pi(j)$ and $i \neq j$. In particular we have $|\ell_t(i) - \ell_t(\pi(i))| \le \epsilon$ and $|\ell_t(j) - \ell_t(\pi(j))| \le \epsilon$ for all $t \in [T]$, and therefore
\[
|\ell_t(i) - \ell_t(j)| \le |\ell_t(i) - \ell_t(\pi(i))| + |\ell_t(\pi(j)) - \ell_t(j)| \le 2 \epsilon~.
\] 
However, there exists $t \in [T]$ such that $|\ell_t(i) - \ell_t(j)| > 2 \epsilon$, thus reaching contradiction.

We now turn to prove the upper bound. Let $S$ be a maximal $\epsilon$-packing, which we will show is also an $\epsilon$-cover.
Suppose otherwise, then there is an expert $i$ such that for every $j \in S$ we have $d(i,j) > \epsilon$. In particular $i \not\in S$.

Consider the set $S' = S \cup \{i\}$ obtained by adding $i$ to $S$. This set is also an $\epsilon$-packing therefore reaching a contradiction.
\end{proof}

\subsection{Technical Lemma}

\begin{lemma}
\label{lemma:logsum}
Let $1 = a_1 < a_2 < \ldots < a_n$ be a sequence of $n$ natural numbers. Then,
\[
\sum_{i=1}^n \log a_i \le 2 a_n \log a_n~.
\]
\end{lemma}

\begin{proof}
Consider the convex function $f(x) = x \log(x) - x$. Then for any $i = 1,2,\ldots,n-1$ we have
\[
f(a_{i+1}) - f(a_{i}) \ge f'(a_i) (a_{i+1} - a_i) = \log(a_i)(a_{i+1} - a_i) \ge \log a_i~.
\]
Summing,
\begin{align*}
\sum_{i=1}^{n-1} \log a_i &\le \sum_{i=1}^{n-1} f(a_{i+1}) - f(a_{i}) = f(a_n) - f(a_1) \\
&= (a_n \log(a_n) - a_n) - (a_1 \log a_1 - a_1) \le a_n \log a_n~,
\end{align*}
and thus
\[
\sum_{i=1}^n \log a_i \le a_n \log a_n + \log a_n \le 2 a_n \log a_n~.
\]
\end{proof}

\end{document}